\documentclass{article}
\usepackage{ijcai23}

\usepackage{times}
\usepackage{soul}
\usepackage{url}
\usepackage[hidelinks]{hyperref}
\usepackage[utf8]{inputenc}
\usepackage[small]{caption}
\usepackage{graphicx}
\usepackage{amsmath}
\usepackage{amsthm}
\usepackage{booktabs}
\usepackage{algorithm}
\usepackage{algorithmic}
\usepackage[switch]{lineno}

\usepackage{amsfonts,amssymb}
\usepackage{float}  
\usepackage{subfigure}  
\usepackage{multirow} 
\usepackage{threeparttable} 
\usepackage{enumitem}
\usepackage{tablefootnote}
\usepackage{tabularx}
\usepackage{balance}
\usepackage{epstopdf}

\newtheorem{theorem}{Theorem}
\newtheorem{corollary}{Corollary}

\title{Action Pick-up in Dynamic Action Space Reinforcement Learning}

\author{
Jiaqi Ye$^1$\and
Xiaodong Li$^1$\and
Pangjing Wu$^{2}$\and
Feng Wang$^{3}$
\affiliations
$^1$College of Computer and Information, Hohai University\\
$^2$Department of Computing, The Hong Kong Polytechnic University\\
$^3$School of Computer Science, Wuhan University
\emails
\{jiaqi.ye, xiaodong.li\}@hhu.edu.cn,
pangjing.wu@polyu.edu.hk,
fengwang@whu.edu.cn
}

\begin{document}

\maketitle

\begin{abstract}
    Most reinforcement learning algorithms are based on a key assumption that Markov decision processes (MDPs) are stationary. However, non-stationary MDPs with dynamic action space are omnipresent in real-world scenarios. Yet problems of dynamic action space reinforcement learning have been studied by many previous works, how to choose valuable actions from new and unseen actions to improve learning efficiency remains unaddressed. To tackle this problem, we propose an intelligent Action Pick-up (AP) algorithm to autonomously choose valuable actions that are most likely to boost performance from a set of new actions. In this paper, we first theoretically analyze and find that a prior optimal policy plays an important role in action pick-up by providing useful knowledge and experience. Then, we design two different AP methods: frequency-based global method and state clustering-based local method, based on the prior optimal policy. Finally, we evaluate the AP on two simulated but challenging environments where action spaces vary over time. Experimental results demonstrate that our proposed AP has advantages over baselines in learning efficiency.
\end{abstract}

\section{Introduction}
Reinforcement learning (RL) has made great achievements in solving sequential decision making problems~\cite{vinyals2019grandmaster,silver2016mastering,silver2017mastering}, and most of existing RL algorithms are based on a key assumption that Markov decision processes (MDPs) are stationary. However, many real-world sequential decision making problems are non-stationary, i.e., transition dynamics~\cite{gajane2018sliding,Variational,cheung2020reinforcement,neu2013online}, reward functions~\cite{gajane2018sliding,Variational,cheung2020reinforcement} and the number of available actions (decisions)~\cite{mandel2017add,boutilier2018planning,chandak2020lifelong} often change over time. For example, it is necessary for designers of intelligent robotics to add new operation components to enhance robotics' ability of interacting with unknown environments; in recommender systems, as new products constantly appear in a market, they are added to the system to improve the recommender accuracy. Examples above are essentially \emph{dynamic action space reinforcement learning (DAS-RL)} problems. When training RL models, DAS-RL brings an inevitable problem: prior network structures and parameters in RL can’t adapt to new action space. If an agent continues making decisions based on the old action space instead of adapting to the changes, the policy is likely to be suboptimal. A trivial solution is to retrain the model according to the new action space available. Given that RL algorithms learn optimal policy by trial-and-error exploration, the whole training process consumes a lot of time and expensive computing resources. If the action space changes frequently, the model needs to be retrained constantly, which is inefficient. Therefore, DAS-RL is worth studying.

Existing researches on solving the DAS-RL problems can be divided into two categories. One is to solve the generalization problem of new actions~\cite{jain2020generalization}, wherein the size of the action space is fixed. What changes is the action itself in the action space, i.e., the action in the new action space is unprecedented. In comparison, the other aims to address the problems where the size of the action space varies over time~\cite{mandel2017add,boutilier2018planning,chandak2020lifelong}. A typical work proposed by Chandak et al.~\shortcite{chandak2020lifelong} combined an action representation function and RL algorithms to address the lifelong learning problem, wherein the size of the action space changes throughout its lifetime. Although the DAS-RL problems have been studied to some extent, we often don’t know whether those new actions bring a positive or negative impact on the policy learning. Selecting valuable actions before RL model training may reduce the probability of the agent exploring less useful actions, which helps improve the learning efficiency of RL. The key consideration here is how to choose valuable actions from a set of new actions.

To address this problem, we propose an intelligent Action Pick-up (AP) algorithm, which automatically selects valuable actions that are most likely to facilitate performance from a set of new actions. We first prove that prior optimal policy is the crux of AP, and the states and actions of the policy can provide vital information in the action pick-up process. Then, based on the prior optimal policy, we design two AP methods: frequency-based global method (AP-FG) and state clustering-based local method (AP-SCL) to select valuable actions. Moreover, we theoretically prove that AP-SCL needs less training time than AP-FG. Finally, we demonstrate the effectiveness of the AP by carrying out experiments on two challenging environments.

Our main contributions are as follows: 1) We prove that the prior optimal policy possesses useful knowledge which can help select valuable actions; 2) Leveraging insights from the conclusion, we creatively design two different AP methods, AP-FG and AP-SCL; 3) We prove that the training speed of AP-SCL is faster than that of AP-FG.


\section{Related Works}\label{sec2}
In non-stationary MDPs, state space often changes over time, i.e., new states haven’t seen before are added to the state space as well as old states gradually fade away. Neu~\shortcite{neu2013online} studied the problem of online learning in non-stationary MDPs where reward function was allowed to change over time. Gajane et al. \shortcite{gajane2018sliding,Variational} considered the situation where both transition dynamics and reward functions varied over time. Chandak et al.~\shortcite{chandak2020optimizing} proposed a policy gradient algorithm that directly fitted a good future policy without modeling transition functions, reward functions, or any other underlying non-stationarity in the environment. It can be seen that the non-stationary MDPs with changing transition dynamics and reward functions have been well-studied.

During the whole sequential decision life, changing of action space can also result in non-stationary MDP. New actions taken by the agent results in the emergence of new states in a stochastic fashion, then the MDPs becomes non-stationary. Many researchers have put forward novel methods to solve DAS-RL problems. Jain et al.~\shortcite{jain2020generalization} introduced zero-shot learning to generalize previously unseen actions, without need to learn from scratch after new actions were available. This work has been limited to the setting where the size of new action space needed to be consistent with the old one, i.e., the size of action space was fixed, and what changes was the action in it. By contrast, following works~\cite{boutilier2018planning,chandak2020lifelong} focused on the problem where the size of the action space changed over time. Boutilier et al.~\shortcite{boutilier2018planning} proposed a stochastic action set MDP (SAS-MDP), which laid a foundation for DAS-RL. In the SAS-MDP, the action space available at each specific state $s$ was a stochastically chosen subset from a fixed, finite number of base actions. In a long episode, there was a possibility that the agent could observe the complete base actions due to the random sampling. Unlike this work, Chandak et al.~\shortcite{chandak2020lifelong} focused on lifelong MDP where new actions were unseen before and the agent could never observe all possible actions. Mandel et al.~\shortcite{mandel2017add} proposed an automated method ELI to intelligently identify states where new actions were most likely to improve performance. This work had some connection with the dynamic action space, but their primary goal was to find the optimal state where a new action should be added.

Chandak et al.~\shortcite{chandak2020lifelong} presented the work most relevant to ours, whose approach focused on the design of an algorithm that could continually adapt to the new actions in DAS-RL, which avoided training from scratch repeatedly. However, in their study, all new actions were added to the old action space. Adding less useful new actions to the action space might increase exploration time and slow down the training speed to a certain extent. Inspired by this practical problem, we propose an action pick-up algorithm to improve the RL training efficiency. Our work also gets inspiration from~\cite{rafati2019learning}. In their work, all states in the discrete state space were clustered and each cluster was given a subgoal, so as to solve the RL problem with sparse rewards. Instead, we focus only on states in the optimal policy, and then cluster these states.

\section{Problem Formulation}\label{sec3}
In this section, we first formulate our non-stationary dynamic action space MDPs, and then discuss the approach to training a RL model in dynamic action space, which is the premise of our AP algorithm.

\subsection{Dynamic Action Space MDPs}
Our formulation of non-stationary \emph{Dynamic Action Space MDPs (DAS-MDPs)} derives from a standard, finite MDP~\cite{sutton2018reinforcement} $\mathcal{M}$ that can be described as a tuple of five elements $<S,A,P,R,\gamma>$. $S$ is a set of all possible states that can be observed by agents in the environment. $A$ consists of a set of fixed, finite discrete actions, called the action space. State transition probability $P:S\times A\rightarrow S$ represents the probability of transition to $s'\in S$ when action $a$ is taken at state $s\in S$ . Reward $R(s,a)$ quantifies the feedback given by the environment after taking action $a$ at state $s$. $\gamma$ is a discount factor that balances instant and future reward. The agent's goal is to maximize the expected discount reward $G(s)=\mathbb{E}[\sum_{t=0}^{T}\gamma^{t}R(s)]$ in a finite horizon.

We define the DAS-MDPs as MDPs $\boldsymbol{\mathcal{M}}=\{\mathcal{M}_k\}_{k=1}^{\infty}$, where each MDP $\mathcal{M}_k$ denotes a standard, finite, stationary MDP in phase $k$. $\mathcal{M}_k=<S,A_k,P,R,\gamma >$ represents state space, action space, state transition probability, reward function and discount factor in $\mathcal{M}_k$, respectively. To simplify the modeling, we assume that $S,P,R,\gamma$ in $\boldsymbol{\mathcal{M}}$ remain the same except the action space $A$. When a set of new actions $\mathcal{A}$ is available, $\mathcal{M}_k$ immediately transitions to $\mathcal{M}_{k+1}$ in which action space $A_{k+1}=A_k \cup \mathcal{A}$. If $\mathcal{A}=\emptyset$, $\mathcal{M}_{k+1}=\mathcal{M}_k$, otherwise $\mathcal{M}_{k+1} \neq \mathcal{M}_k$. The objective of the agent is to quickly adapt to the new action space $A_{k+1}$, and maximize the expected discount reward in $\mathcal{M}_{k+1}$.

\subsection{Learning in Dynamic Action Space}\label{sec:3.2}
For RL with discrete action space, the output is the probability or $Q$-value of each action, where the number of parameters that the output layer has is dependent on the size of action space. Once additional actions are added to the action space, the number of parameters must be modified to adapt to the change, which means we need to retrain the model. \emph{How can we avoid resources wasting caused by retraining RL models in the setting of discrete action space?} To tackle this problem, we get inspiration from RL paradigms with continuous action space~\cite{sutton2018reinforcement} and action representation methods~\cite{chandak2019learning,dulac2015deep}. In RL with continuous action space, the model outputs $e\in \delta $ corresponding to a $d$-dimensional vector, according to the policy $\pi_{\theta}:S\times \delta \rightarrow [0,1]$, where $\delta \in \mathbb{R}^d$. As long as the dimension $d$ is determined, the number of parameters is fixed. After that, the action representation function maps a continuous action taken by RL to a discrete action.

In $\mathcal{M}_k$, RL model outputs low-dimensional continuous action $e_k \in \delta_k$, where $\delta_k \in \mathbb{R}^d$, according to policy $\pi_{\theta}^k$. It is noteworthy that dimension $d$ is fixed during $\boldsymbol{\mathcal{M}}$. Then, action representation function~\cite{chandak2019learning} $f_k:\delta_k \rightarrow A_k$, maps $e_k$ to discrete action $a_k \in A_k$, i.e., $f_k(e_k)=a_k$. The agent takes action $a_k$ to interact with the environment, and continuously trains policy $\pi_{\theta}^k$ until it converges to the optimal policy $\pi_{\theta^*}^k$. After transitioning to $\mathcal{M}_{k+1}$, the optimal policy $\pi_{\theta^*}^k$ in $\mathcal{M}_k$ is carried over, i.e., $\pi_{\theta}^{k+1}=\pi_{\theta^*}^k$. The prime benefit of doing so is to avoid retraining the model and reuse prior knowledge. In addition, we only need to modify the parameters of $f_k$ and then train the new function $f_{k+1}$ via history experience trajectory. By doing so, the agent can adapt to the new action space efficiently.

\section{Intelligent Action Pick-up}\label{sec4}
If a set of valuable actions $\mathcal{A}^*(\mathcal{A}^* \subset \mathcal{A})$ can be accurately selected from new actions $\mathcal{A}$, the RL learning efficiency will be further improved. This raises a natural question: \emph{how to select valuable actions $\mathcal{A}^*$?} To answer this question, we propose an action pick-up algorithm to intelligently obtain $\mathcal{A}^*$ by solving the objective function defined in Section \ref{sec:4.1}. Details of the action pick-up algorithm are available in Appendix A.1.
\begin{figure}[htb]
    \centering
    \includegraphics[width=0.9\linewidth,height=0.6\linewidth]{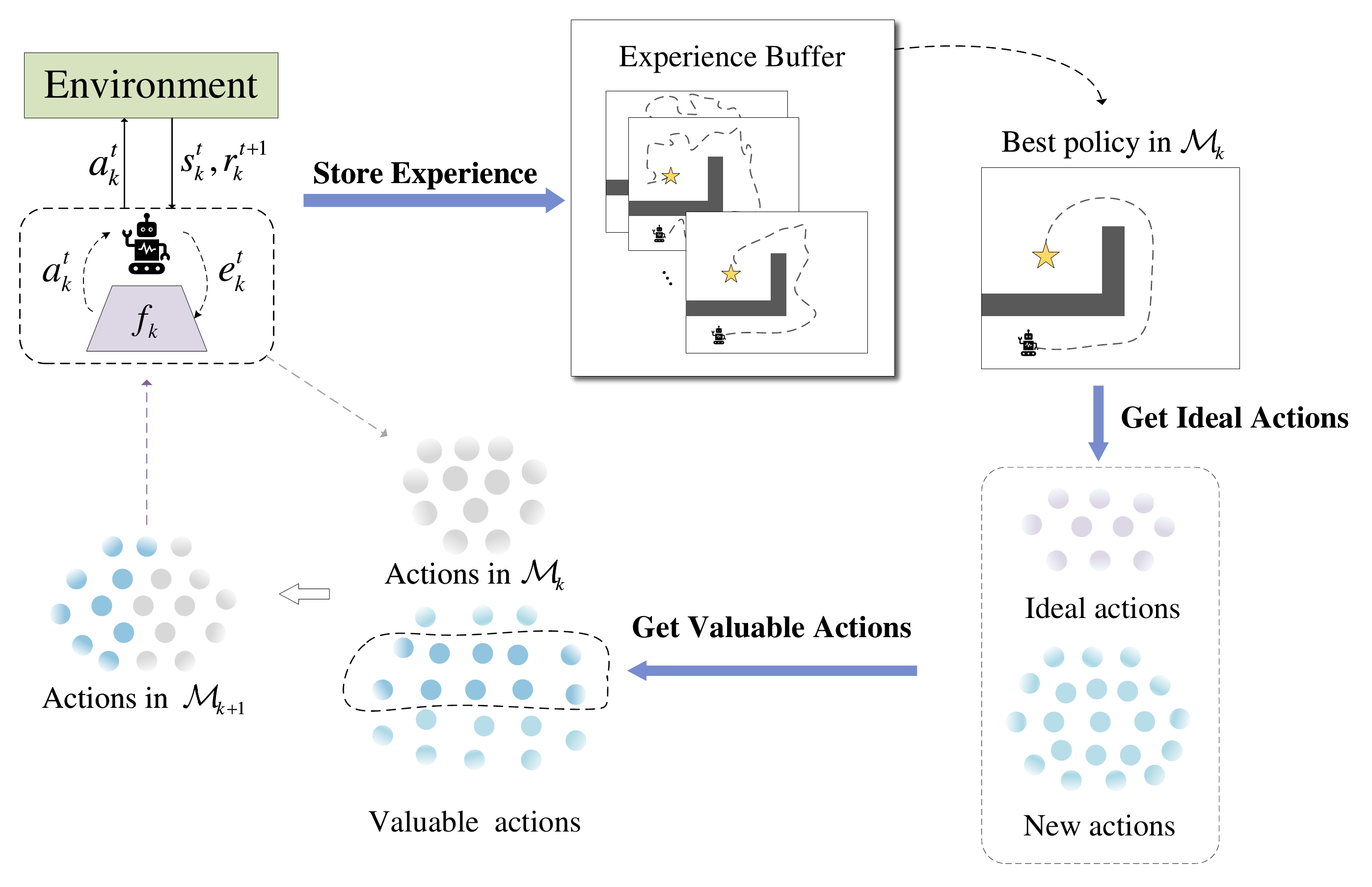}
    \caption{Action pick-up framework for selecting valuable actions. When MDP $M_k$ transitions to $M_{k+1}$, the action space $A_{k+1}$ in $M_{k+1}$ differs from $A_{k}$, and $A_{k+1}$ is derived from $A_{k}$, i.e., $A_{k+1}=A_k \cup \mathcal{A}^*$.} 
    \label{fig:AP}
\end{figure}

\subsection{Objective Function}\label{sec:4.1}
For $a \in \mathcal{A}$, its value $V(a)$ can be defined as follows:
\begin{equation}
    V(a):=\mathop{\rm max}\limits_{\tilde{a} \in \mathcal{\tilde{A}}}\ sim(a,\tilde{a}),
\end{equation}
where $\Tilde{a}$ denotes the action in the ideal action set $\Tilde{\mathcal{A}}$, and $sim$ represents a similarity measure function. The greater the value of $sim(a,\Tilde{a})$ is, the more similar $a$ and $\Tilde{a}$ is. 
Our goal is to select the action with the highest $V(a)$, so the objective function can be defined as:
\begin{align}
    a=&\mathop{\rm argmax}\limits_{a \in \mathcal{A}}\ V(a)\nonumber\\
    =&\mathop{\rm argmax}\limits_{a \in \mathcal{A}}\ \mathop{\rm max}\limits_{\Tilde{a} \in \Tilde{\mathcal{A}}}\ sim(a,\Tilde{a}).
    \end{align}
We can select a set of valuable actions $\mathcal{A}^*$, by calculating the value of each action $a \in \mathcal{A}$. 

The ideal action set $\Tilde{\mathcal{A}}$ needs to be obtained before solving the objective function. We believe that the optimal policy $\pi_{\theta^*}^{k}$ in $\mathcal{M}_k$ has useful prior knowledge, and it is convincing to obtain the $\Tilde{a}$ according to the state and action information in the $\pi_{\theta^*}^{k}$. Here, we give a generalized definition of ideal action $\Tilde{a}$,
\begin{equation}\label{ideal-actions}
    \Tilde{a}:=\phi((\hat{a}_i,\hat{a}_j),(w(\hat{a}_i),w(\hat{a}_j))),
\end{equation}
where the input of $\phi$ is an action pair $(\hat{a}_i,\hat{a}_j)$ and their corresponding weight $w((\hat{a}_i),w(\hat{a}_j))$. $\hat{a}_i$ and $\hat{a}_j$ are taken from $\hat{A}$ consisting of all actions in the optimal policy. Specific theoretical derivation that proves the rationality of the design of $\Tilde{a}$ is in Section \ref{sec4.2}.

\subsection{Theoretical Analysis}\label{sec4.2}
Let $\boldsymbol \pi$ denotes a set of all policies learned by the agent in $\mathcal{M}_k$, and $\boldsymbol \pi^*$  denotes the optimal policies in $\mathcal{M}_k$. Since there may be more than one optimal policy in $\mathcal{M}_k$, the size of $\boldsymbol \pi^*$ is greater than or equal to $1$, i.e., $|\boldsymbol{\pi}^*| \ge 1$.
\begin{theorem}\label{the1}
    Actions on trajectory $\tau=\{s_0,a_0,r_1,s_1,a_1,r_2,\\s_2,a_2,r_3,\cdots,s_{T^*-1},a_{T^{*}-1},r_{T^*}\}$ generated by optimal policy $\pi^* \ (\pi^*\in \boldsymbol{\pi}^*)$ have the highest value than those on trajectories generated by $\pi\ (\pi \in \boldsymbol{\pi-\pi}^*)$.
\end{theorem}
\begin{proof}
    The relationship between state-action value function $q(s,a)$ and state value function $V(s)$ can be represented by $q_\pi(s,a)=\mathbb{E}[R_{t+1}+\gamma v_\pi(s')|S_t=s,A_t=a]$. See Eq. (\ref{q-value}) for detailed derivation. 
    \begin{align}\label{q-value}
     q_\pi(s,a)= & \sum_{s'}\sum_{r}p(s',r|s,a)(r+\gamma v_\pi(s'))\nonumber\\
     = & \resizebox{.8\linewidth}{!}{$\displaystyle \sum_{s'}\sum_{r}rp(s',r|s,a)+\sum_{s'}\sum_{r}p(s',r|s,a)\gamma v_\pi(s')\nonumber$}\\
     =&\resizebox{.8\linewidth}{!}{$\displaystyle  \mathbb{E}[R_{t+1}|S_t=s,A_t=a]+\gamma \mathbb{E}[v_\pi(s')|S_t=s,A_t=a]\nonumber$}\\
     = & \mathbb{E}[R_{t+1}+\gamma v_\pi(s')|S_t=s,A_t=a].
    \end{align}
     In the situation where $\pi^*$ is better than $\pi$, the value function of $\pi^*$ is always better than that of $\pi$, i.e., $v_{\pi^*}(s)>v_\pi(s)$ for any $s \in S$. According the following conditions,
    \begin{equation}
        \left\{\begin{array}{l}
             v_{\pi^*}(s)>v_\pi(s)\\
             q_\pi(s,a)=\mathbb{E}[R_{t+1}+\gamma v_\pi(s')|S_t=s,A_t=a]\\
            q_{\pi^{*}}(s,a)=\mathbb{E}[R_{t+1}+\gamma v_{\pi^*}(s')|S_t=s,A_t=a]
        \end{array}\right.,
    \end{equation}
we can draw a conclusion that $q_{\pi^*}(s,a) > q_\pi(s,a)$ at any state $s$.
\end{proof}

Theorem \ref{the1} presents an important conclusion that the $q_{\pi^{*}}(s,a)$ in the optimal policy is the highest, so it is reasonable to obtain $\Tilde{\mathcal{A}}$ by using the state and action information provided by the optimal policy. In $\mathcal{M}_k$, we obtain lots of policies. \emph{How to efficiently obtain the optimal policy among those policies?} The following theorem and corollary provide us a theoretical basis for obtaining the optimal policy.

\begin{theorem}\label{the2}
    In a series of tasks $\mathcal{T}$ where an agent aims at reaching the goal $\mathcal{G}$ from initial state $s_0$ as fast as possible, and $T$ denotes total steps the agent needs in an episode. If policy $\pi_i$ is better than $\pi_j$, then the total steps $T_i$ the agent needs in $\pi_i$ is fewer than $T_j$, i.e., if $v_{\pi_i}(s_0)>v_{\pi_j}(s_0)$, then $T_i<T_j$.
\end{theorem}

\begin{proof}
For any state $s \in S$, its state value function is
    \begin{equation}
    \resizebox{.89\linewidth}{!}{$
            \displaystyle
           V_{\pi}(s) = \mathbb{E}_\pi [G_t|S_t=s] = \mathbb{E}_\pi [\sum_{k=0}^{T}\gamma^{k}R_{t+k+1}|S_t=s]
        $},
    \end{equation}%
then the state value function at the initial state $s_0$  can be represented by:
    \begin{equation}
        V_{\pi}(s_0) = \mathbb{E}_\pi[\sum_{k=0}^{T}\gamma^{k}R_{k+1}|s_0].
    \end{equation}
    In tasks $\mathcal{T}$, in order to encourage the agent to reach the goal $\mathcal{G}$ as quickly as possible, it is assumed that punishment $R^{(1)} (R^{(1)} \le 0, \ R^{(1)} \mbox{\ is not mandatory})$ is given after each step. When the agent reaches $\mathcal{G}$, it gets a reward $R^{(2)}$. The reward function is defined as
    \begin{equation}
        R=\left\{
        \begin{aligned} 
          R^{(1)}, & \mbox{\quad the penalty given at each time step}\\
          R^{(2)}, & \mbox{\quad the reward for reaching the goal}
        \end{aligned}
        \right.. 
    \end{equation}
    Let $G(T)=\sum_{k=0}^{T-1}\gamma^{k}R^{(1)}+\gamma^{T}R^{(2)}$. Without loss of generality, $\gamma$, $R^{(1)}$ and $R^{(2)}$ are constants.

    \begin{align}
     G(T) = &\sum_{k=0}^{T-1}\gamma^{k}R^{(1)}+\gamma^{T}R^{(2)}\nonumber\\
     = & R^{(1)}(\gamma^0+\gamma^1+\cdots+\gamma^{T-1})+\gamma^TR^{(2)}\nonumber\\
     = & R^{(1)}\frac{1-\gamma^{T}}{1-\gamma}+\gamma^TR^{(2)}.
    \end{align}
    Take derivative of $G(T)$ with respect to $T$,
    \begin{align}
     \frac{\mathrm{d} G(T)}{\mathrm{d} T} = & R^{(1)}\frac{-\gamma^Tln\gamma}{1-\gamma}+R^{(2)}\gamma^Tln\gamma \nonumber\\
     = & \gamma^Tln\gamma(\frac{R^{(1)}}{\gamma-1}+R^{(2)}).
    \end{align}
    Since $\gamma \in (0,1)$, we have $\gamma^Tln\gamma<0$, hence the monotonicity of $G(T)$ is decided by polynomial $\frac{R^{(1)}}{\gamma-1}+R^{(2)}$.
    \begin{itemize}
        \item When $\frac{R^{(1)}}{\gamma-1}+R^{(2)}>0$, i.e., $R^{(1)}<(1-\gamma)R^{(2)}$, $G^{'}(T)<0$, then $G(T)$ is monotonically decreasing.
        \item When $\frac{R^{(1)}}{\gamma-1}+R^{(2)}<0$, i.e., $R^{(1)}>(1-\gamma)R^{(2)}$, $G^{'}(T)>0$, then $G(T)$ is monotonically increasing.
    \end{itemize}
    In tasks $\mathcal{T}$, according to the following conditions
        \begin{equation}
        R=\left\{
        \begin{array}{l}
          0<(1-\gamma)R^{(2)}<R^{(2)}\\
          R^{(1)} \le 0
        \end{array}
        \right.,
    \end{equation}
    $R^{(1)}<(1-\gamma)R^{(2)}$ is always true, hence $G(T)$ is a monotonically decreasing function and is lower bounded by $\frac{R^{(1)}}{\gamma-1}$. See Eq. (\ref{lowerbound}) for detailed derivation.
    \begin{equation}\label{lowerbound}
        \lim_{T \to \infty}G(T)=\lim_{T \to \infty} R^{(1)}\frac{1-\gamma^T}{1-\gamma}+\gamma^TR^{(2)}=\frac{R^{(1)}}{1-\gamma}.  
    \end{equation}
    For policy $\pi$ and $\pi'$, the state value function in state $s_0$ is $v_\pi(s_0)$ and $v_{\pi'}(s_0)$, respectively.
    \begin{equation}
        v_\pi(s_0)= \mathbb{E}_\pi [\sum_{k=0}^{T}\gamma^{k}R_{k+1}|s_0]=\mathbb{E}_\pi[G(T)|s_0],
    \end{equation}
        \begin{equation}
        v_{\pi'}(s_0)= \mathbb{E}_{\pi'} [\sum_{k=0}^{T'}\gamma^{k}R_{k+1}|s_0]=\mathbb{E}_{\pi'}[G(T')|s_0].
    \end{equation}
    If policy $\pi$ is better than $\pi'$, for all $s\in S$, $v_\pi(s)>v_{\pi'}(s)$ is always true, so $v_\pi(s_0)>v_{\pi'}(s_0)$, i.e., $\mathbb{E}_\pi[G(T)|s_0]>\mathbb{E}_{\pi'}[G(T')|s_0]$. We assume that the expectation is calculated after one sampling using first-visit Monte Carlo sampling~\cite{sutton2018reinforcement}, hence $\mathbb{E}_\pi[G(T)|s_0]=G(T)$. $\mathbb{E}_\pi[G(T)|s_0]>\mathbb{E}_{\pi'}[G(T')|s_0]$ is equivalent to $G(T)>G(T')$. According to the decreasing monotonicity of $G(T)$, $T_i<T_j$. 
\end{proof}
\noindent Using Theorem \ref{the2}, we can get the following corollary.

\begin{corollary}\label{colo}
Let $\mathbb{T}=\{T_{i}|i=1,2,3,…,m\}$, where $T_i$ denotes the steps needed by $\pi_{i}$ to get the goal $\mathcal{G}$. If $T_j$ is the smallest in $\mathbb{T}$, then the policy $\pi_j \in \boldsymbol{\pi}^*$.
\end{corollary}

\begin{proof}
    
    

    See Appendix A.2.
\end{proof}
\noindent In $\mathcal{M}_k$, the policy in which the agent requires the fewest steps to get to the goal is the optimal policy $\pi_{\theta^*}^k$.

\subsection{Ideal Action Acquisition}
It is proved in Theorem \ref{the1} that the optimal policy has prior knowledge, and it is reasonable to use the state and action information in the optimal policy to obtain $\Tilde{\mathcal{A}}$. For the trajectory $\tau=\{s_0,a_0,r_1,s_1,a_1,r_2,\cdots,s_{T^*-1},a_{T^{*}-1},r_{T^*}\}$ of the optimal policy $\pi_{\theta^*}^k$ in $\mathcal{M}_k$, we first put all the actions in $\tau$ into a set $\hat{A}$. Then we calculate the frequency $f(a_i)$ of each action $a_i \in \tau$. The frequency of an action can be regarded as its weight. To ensure the weight is between 0 and 1, it is normalized by Eq. (\ref{weight}).
\begin{equation}\label{weight}
    w(a_i)=\frac{f(a_i)}{\mathop{\rm max}\limits_{a_i \in \tau}f(a_i) }.
\end{equation}
We finally design two different methods to obtain the $\Tilde{\mathcal{A}}$, as shown in Figure \ref{fig:generation}.
\begin{figure}[htb]
    \centering
    \includegraphics[width=0.9\linewidth,height=0.6\linewidth]{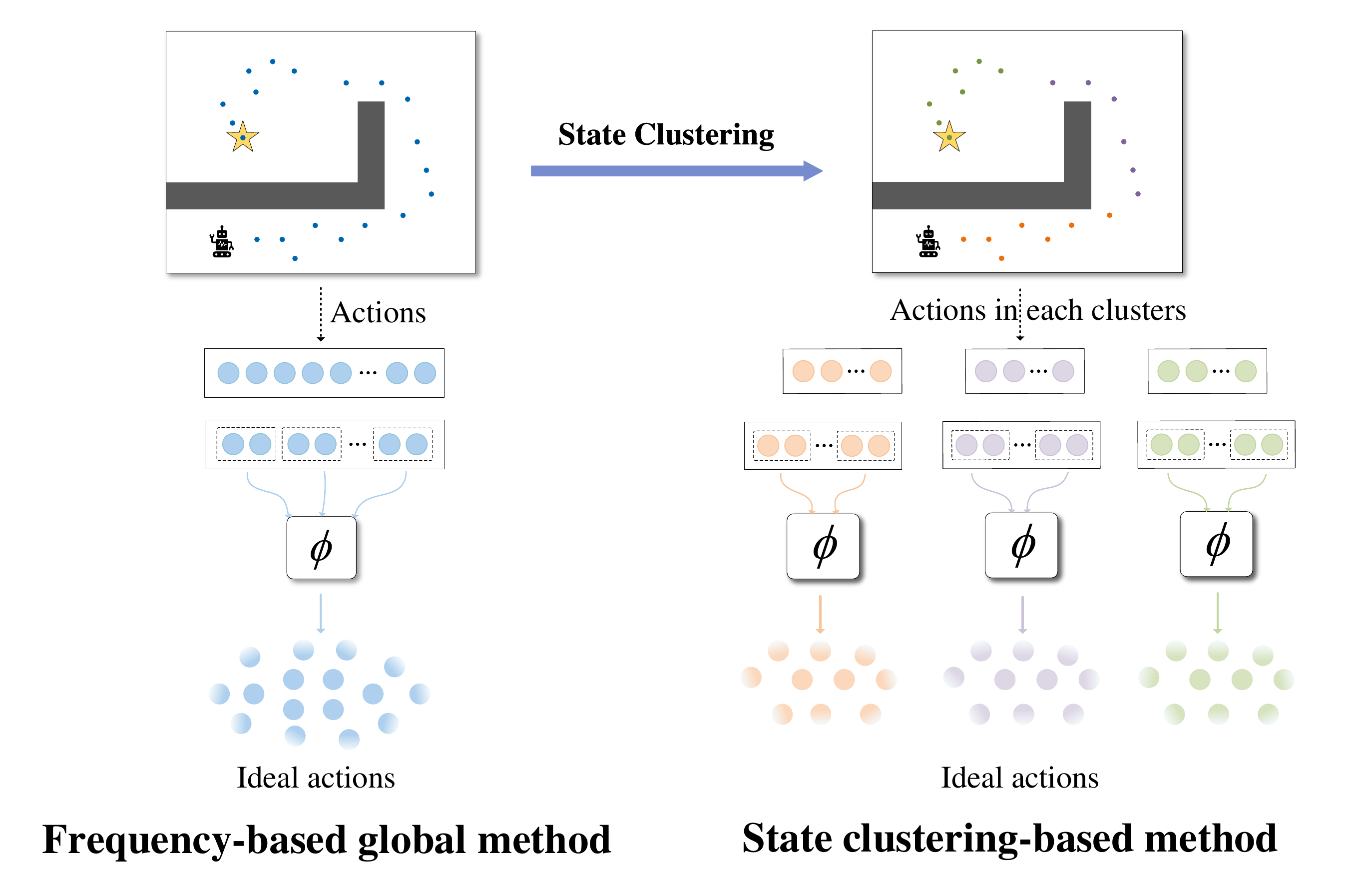}
    \caption{The framework for obtaining ideal actions.}
    \label{fig:generation}
\end{figure}

\subsubsection{Frequency-based Global Method (AP-FG)}
To explore effects of the AP method in which only action information is used to obtain the $\hat{A}$, we design the AP-FG. Two different actions $\hat{a}_i$ and $\hat{a}_j$ in $\hat{A}$ are taken to obtain one ideal action $\Tilde{a}$, according to the definition of $\Tilde{a}$ shown in the Eq. (\ref{ideal-actions}). There are total $C_{|\hat{A}|}^2$ ideal actions can be obtained in this setting.
\subsubsection{State Clustering-based Local Method (AP-SCL)}
In the AP-FG, only the action information in the optimal policy is used, ignoring the state information. In contrast, we consider both state and action information in the AP-SCL. It firstly uses clustering method to cluster $\hat{S}=\{s_0,s_1,\dots,s_{T^*-1}\}$ into $k$ clusters. Then, take the corresponding actions of the states in cluster $c_i$ together, and put them in $\hat{A}_i$. Finally, for each action set $\hat{A}_i$, AP-FG is used to obtain the ideal action subset $\Tilde{
\mathcal{A}}_i$, and the ultimate $\Tilde{\mathcal{A}}=\{\Tilde{
\mathcal{A}}_i|i=1,2,\dots,k\}$. Compare the design philosophy of AP-SCL with that of AP-FG, an advantage of AP-SCL is that it can effectively improve the training speed when $|\hat{A}|$ is large and the proof is as follows.
\begin{proof}
    In the AP-FG, the number of action pairs $(\hat{a}_i,\hat{a}_j)$ is $C_{|\hat{A}|}^2$, and the recursion required to compute $C_{|\hat{A}|}^2$ is
    \begin{equation}
    \resizebox{0.98\linewidth}{!}{$
        H(|\hat{A}|,2)=\left\{
        \begin{array}{l}
          1, \quad if \quad |\hat{A}|=2\\
         H(|\hat{A}|-1,2)+H(|\hat{A}|-1,1)+1,  otherwise
        \end{array}
        \right.. $}
    \end{equation}
    Let $ S(|\hat{A}|,2)=H(|\hat{A}|,2)+1$, then
        \begin{equation}
        S(|\hat{A}|,2)=\left\{
        \begin{array}{l}
          2, \quad if \quad |\hat{A}|=2\\
          S(|\hat{A}|-1,2)+S(|\hat{A}|-1,1),  otherwise
        \end{array}
        \right.. 
    \end{equation}
    Easy to know $S(|\hat{A}|,2)=2C_{|\hat{A}|}^2$, hence 
    \begin{equation}\label{timesfg}
        H(|\hat{A}|,2)=2C_{|\hat{A}|}^2-1.
    \end{equation}
    In the AP-SCL, $|\hat{A}|= {\textstyle \sum_{i=1}^{K}|\hat{A}_i|} $, the number of recursions required in cluster $c_i$ is $ H(|\hat{A}_i|,2)=2C_{|\hat{A}_i|}^2-1$. The total number of recursions is 
    \begin{equation}\label{timescl}
    \sum_{i=1}^{K}H(|\hat{A}_i|,2)=\sum_{i=1}^K(2C_{|\hat{A}_i|}^2-1).
    \end{equation}

    \noindent Expanding Eqs. (\ref{timesfg}) and (\ref{timescl}), we can get the following results,
     \begin{equation}
        \left\{\begin{array}{l}
             H(|\hat{A}|,2)={|\hat{A}|}^2-|\hat{A}|-1\\
             \sum_{i=1}^{K}H(|\hat{A}_i|,2)=\sum_{i=1}^K{|\hat{A}_i|}^2-|\hat{A}|-N
        \end{array}\right.,
    \end{equation}
    Because ${|\hat{A}|}^2={(\sum_{i=1}^K|\hat{A}_i|)}^2>\sum_{i=1}^K{|\hat{A}_i|}^2$, $H(|\hat{A}|,2)>\sum_{i=1}^KH(|\hat{A}_i|,2)$.
    
\end{proof}

We can draw a conclusion that the number of iterations required in the AP-SCL is less than that in the AP-FG. When $|\hat{A}|$ is large, the AP-SCL can effectively improve the training speed from the following aspects: 1) the iteration time of obtaining ideal actions; 2) the calculation time of $V(a)$. After obtaining the $\Tilde{\mathcal{A}}$, we can select the valuable actions by solving the objective function.
\section{Experiment and Results}\label{sec5}
In this section, we compare the performance of our AP methods to baselines in two simulated but challenging environments.
\subsection{Environments}
\subsubsection{Maze Domain}
This is a maze environment with continuous states and the state space is comprised of the coordinates of the agent’s location. The agent needs to reach the goal as soon as possible while avoiding obstacles. In order to encourage the agent to reach the goal, the agent will receive a penalty of -0.05 for each step, and get a reward of 100 after reaching the goal. Additional environment details are available in Appendix B.1.
\subsubsection{Modified Frozen Lake}

In this environment, the agent needs to reach the destination from the initial position as soon as possible without falling into ice holes. Once it falls into the hole, the agent will go back to the initial position and start exploring again. The state of the environment is changed from the original discrete state to the continuous state. The action is represented by the coordinate $(x,y)$, and each action reflects the direction and the length of the movement. In order to better simulate the slippery nature of the frozen lake, we add random disturbance to simulate the uncertainty of the action every time the agent takes the action. More details are provided in Appendix B.1. 

\subsection{Baselines and Evaluation Metrics}
To empirically demonstrate the effectiveness of our methods, three baselines are used in the experiments.
    \begin{itemize}[leftmargin=*]
       \item LAICA(AC): The method proposed in~\cite{chandak2020lifelong}, using Actor-Critic algorithm to optimize policy. New actions are all added into action space instead of selecting valuable actions.
        \item LAICA(DPG): A variation of LAICA(AC), using Deterministic Policy Gradient (DPG) algorithm to optimize policy.
        \item Random Selection Method: Given that our methods leverage action pick-up technology, the size of action space in AP is smaller than that in LAICA. To avoid fast convergence caused by a small size of action space, we set this random selection method as a comparison. It randomly selects several actions from a set of new actions, and the number of selected actions is consistent with that in AP.
    \end{itemize}
We use the total expected reward curve during learning process to measure the performance of different algorithms. Naturally, the reward is prone to be fluctuate, especially at the initial episodes. For the sake of observation and comparison, the curve is smoothed by Exponential Moving Average (EMA).
\vspace{5pt}
\subsection{Experiment Setup}\vspace{5pt}
\subsubsection{General Setting}\vspace{5pt}
The general setting of LAICA is consistent with that in \cite{chandak2020lifelong}, apart from a few hyper-parameters. Our AP methods is divided into two phases. In the initial phase $\mathcal{M}_0$, prior optimal policy is unavailable, so we cannot select valuable actions. Since the agents in the two environments are controlled by actuators, the outcome of an action is associated with the selected actuators. Turning on different actuators can result in the same outcome, hence we only remove repetitive actions from action space $A_0$ in the $\mathcal{M}_0$. After $\mathcal{M}_0$, a large number of policies can be obtained. Based on Corollary \ref{colo}, the policy with the fewest steps is the optimal policy. Combined with the optimal policy available, we apply our AP methods to select valuable actions from a set of new actions and add them to the old action space.

\subsubsection{Action Pick-up Setting}\vspace{5pt}
Our two simulated environments can be abstracted as navigation problems, wherein the function $\phi$ is designed as follows:
\begin{equation}
    \phi((\hat{a}_i,\hat{a}_j),(w(\hat{a}_i),w(\hat{a}_j)))=w(\hat{a}_i)\times \hat{a}_i+w(\hat{a}_j) \times \hat{a}_j.
\end{equation}
Considering that the action is represented by two-dimensional coordinates, the similarity measure function $sim$ is designed as follows:
\begin{equation}
    sim(a,\tilde{a})=\frac{1}{1+\sqrt{(x_a-x_{\tilde{a}})^{2}+(y_a-y_{\tilde{a}})^2}}.
\end{equation}
In the actual experiments, instead of selecting only one action with the highest value, we design an action pick-up rule: calculate the average value $\overline{V(a)}$ of all actions in $\mathcal{A}$, and actions with higher value than $\overline{V(a)}$ are added to the $\mathcal{A}^*$. Both environments use this rule to select valuable actions. In the AP-SCL, K-means is used to cluster the states on the optimal policy.

\subsubsection{Hyper-parameters Setting}
The hyper-parameters settings in the experiments are shown in Table \ref{table:hyper}. Apart from the hyper-parameters listed below, others are consistent with those in \cite{chandak2020lifelong}. The hyper-parameters settings apply to baselines and our AP methods. When training the action representation function, the convergence condition in our methods is more relaxed compared with baselines. Full details are available in Appendix B.2.

\begin{table}[!thb]
\scriptsize
\centering
\renewcommand{\arraystretch}{1.5}
\setlength{\tabcolsep}{2mm}{
\begin{tabular}{cccc}
\toprule
\multirow{2}{*}{Environments} & \multicolumn{3}{c}{Hyper-parameters}                   \\ \cline{2-4} 
& Episodes & Change times\tablefootnote{The number of changes in action space.} & Number of state-action pairs\tablefootnote{State-action pairs are used as samples to train action representation function after every change in action space.} \\ \cline{1-4} 
Maze(AC)& 10000&4&10000\\
Maze(DPG) & 15000  & 4            & 50000  \\
Frozen Lake\tablefootnote{The settings in Frozen Lake(AC) and Frozen Lake(DPG) are the same.} & 10000& 4  & 10000\\ 
\bottomrule
\end{tabular}}
\caption{Hyper-parameters settings in different environments.}
\label{table:hyper}
\end{table}

\subsection{Experimental Results}
\begin{figure*}[!hbt]
        \setlength{\belowcaptionskip}{-5pt}
	\centering  
	\subfigbottomskip=-3pt 
	\subfigcapskip=-2pt 
	\subfigure[]{
		\includegraphics[width=0.41\linewidth]{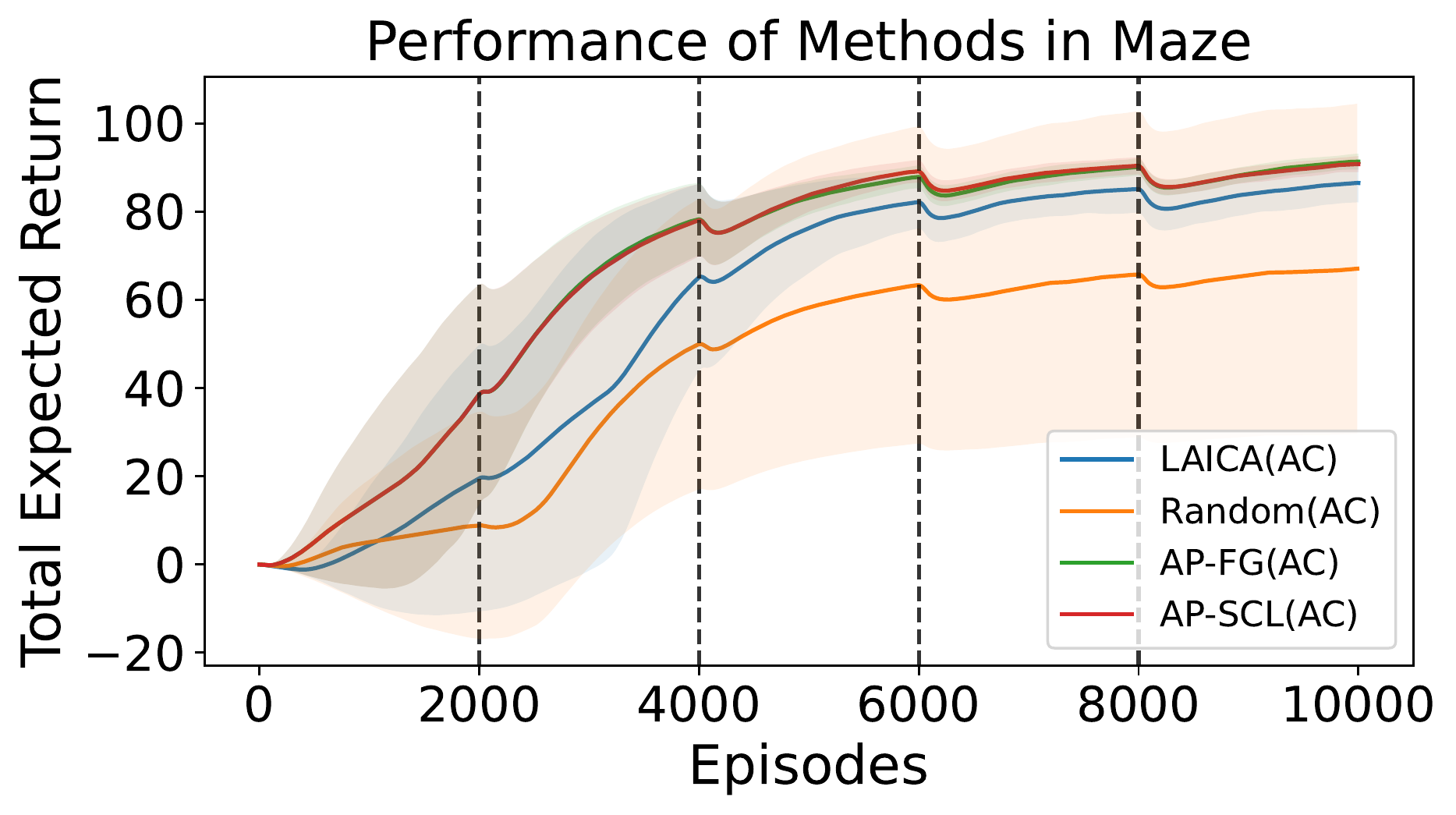}}
        \hspace{1pt}
	\subfigure[]{
		\includegraphics[width=0.41\linewidth]{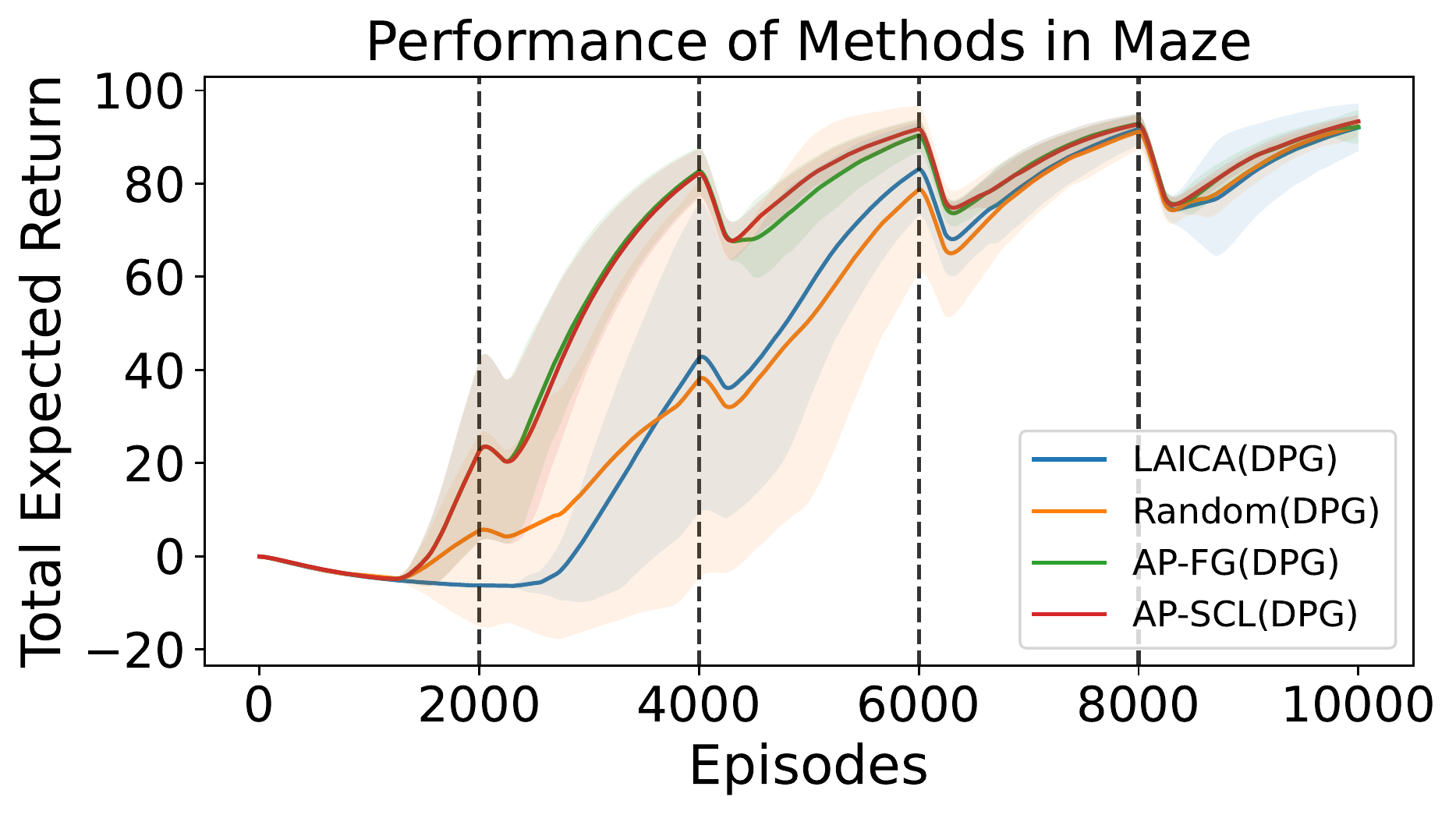}}
	  \\
	\subfigure[]{
		\includegraphics[width=0.41\linewidth]{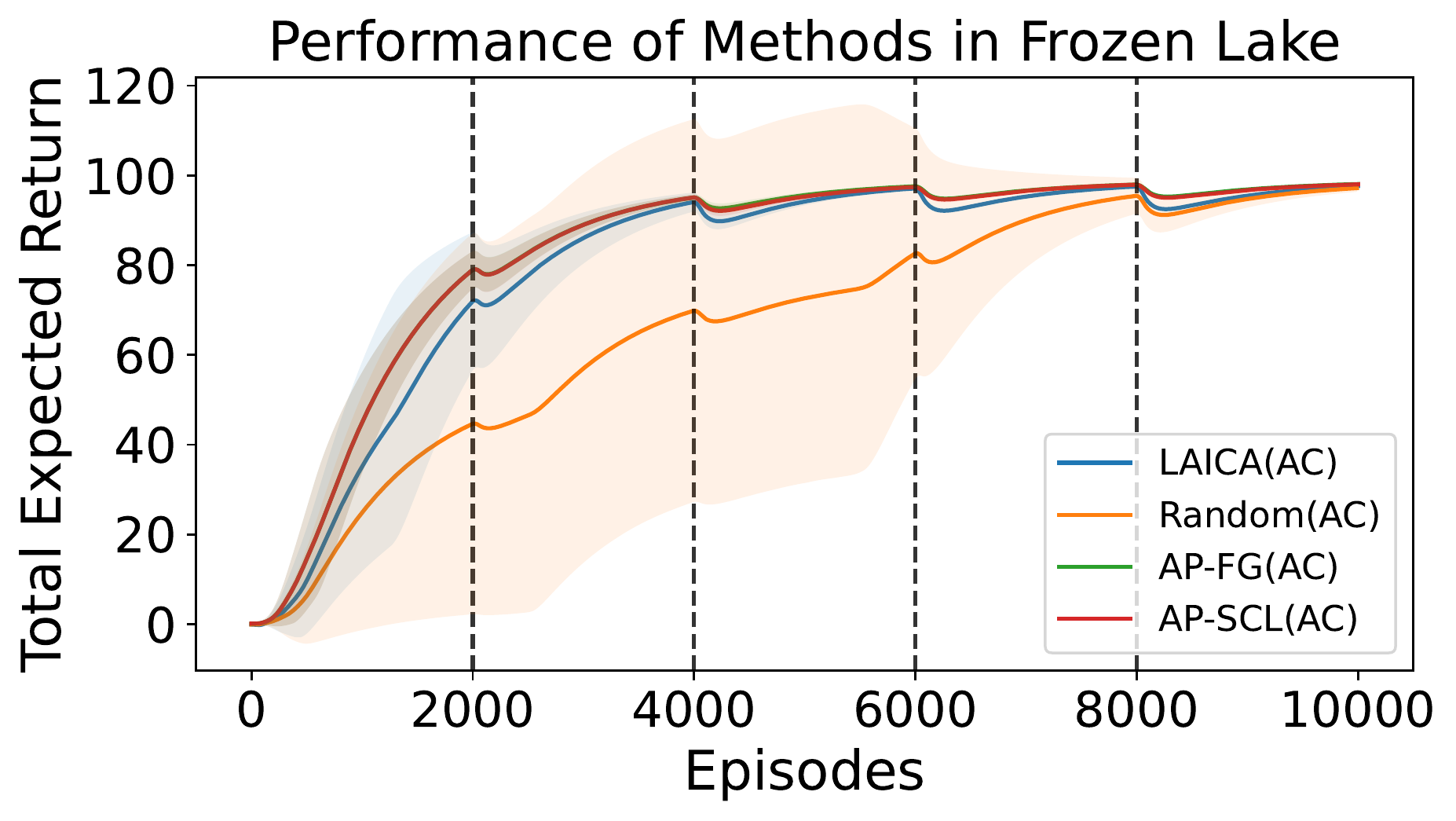}}
	\hspace{1pt}
	\subfigure[]{
		\includegraphics[width=0.41\linewidth]{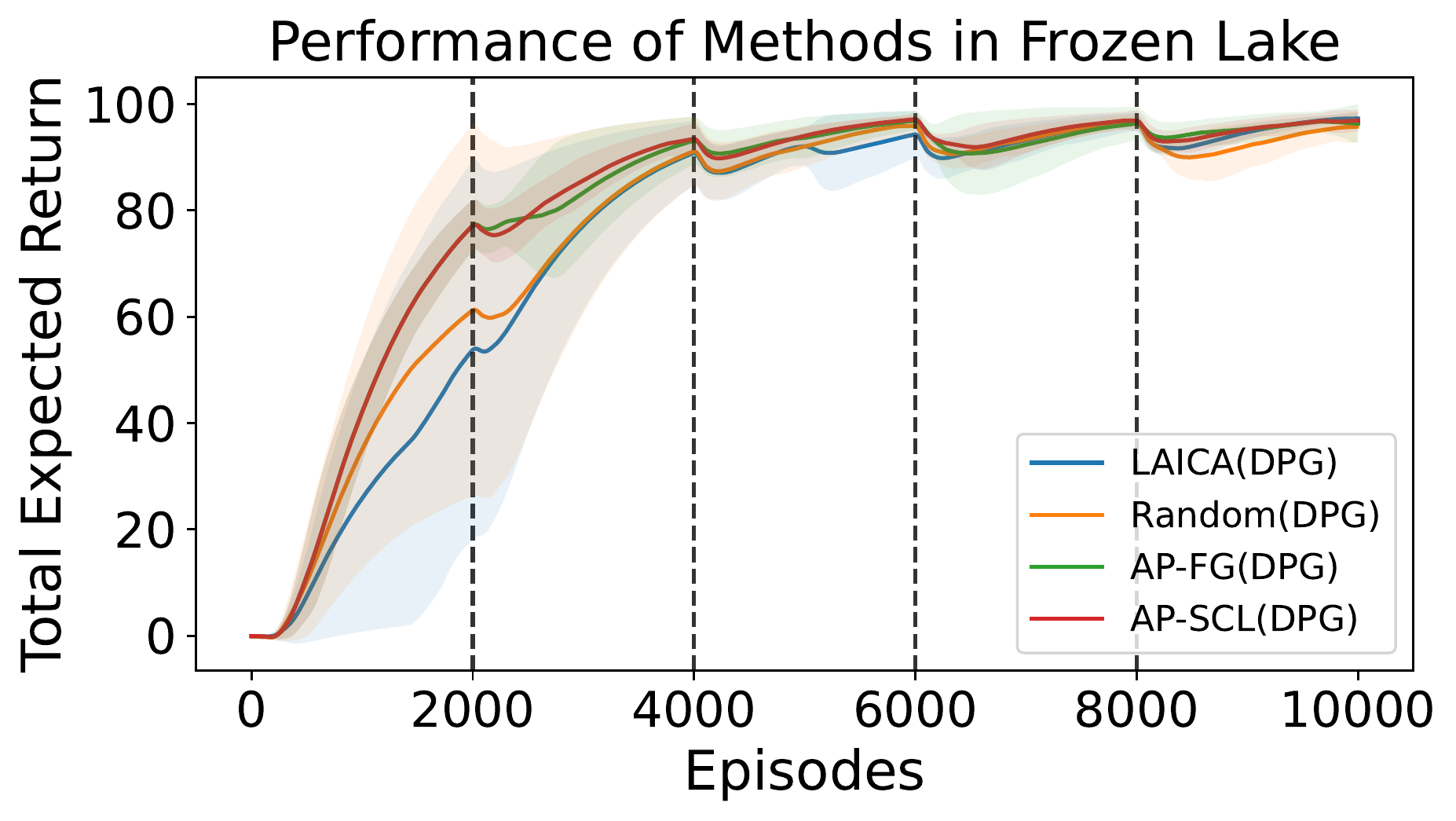}}
	\caption{Performance between baselines and our AP in the Maze and Frozen Lake environments. For each configuration, we perform five random trails in the best hyper-parameters settings. The learning curves correspond to the average results of each method running with five different seeds. The shaded regions correspond to standard error obtained using five trials. The vertical dotted bars indicate entering the next MDP, at which point the action space changes. In the AP-SCL, $k=3$.}

        \label{fig:all}
\end{figure*}

\begin{figure}[!hbt]
        \setlength{\belowcaptionskip}{-5pt}
	\centering  
	\subfigbottomskip=1pt 
	\subfigcapskip=-4pt 
	\subfigure[]{
		\includegraphics[width=0.49\linewidth]{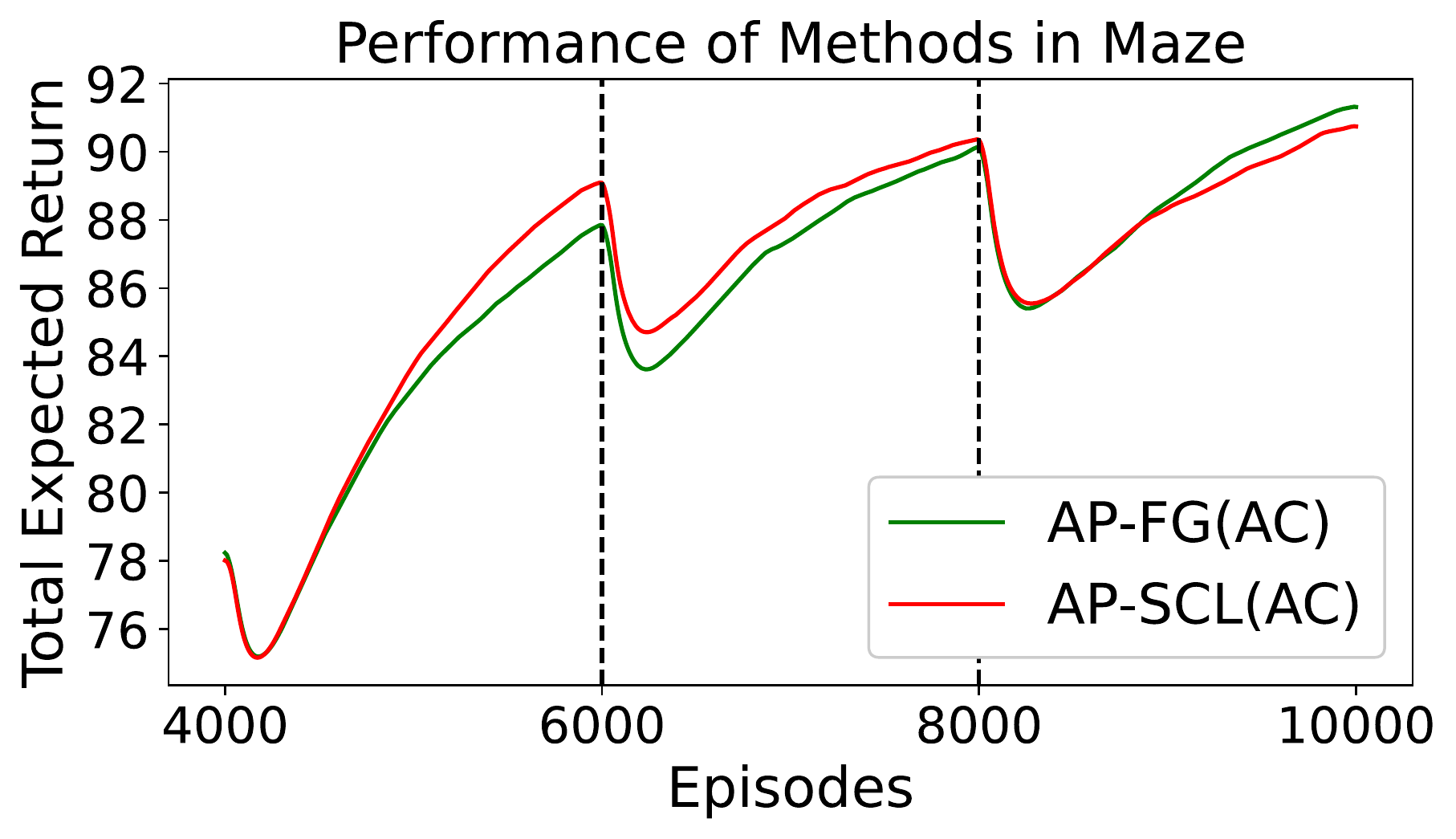}}
        \hspace{-6pt}
	\subfigure[]{
		\includegraphics[width=0.49\linewidth]{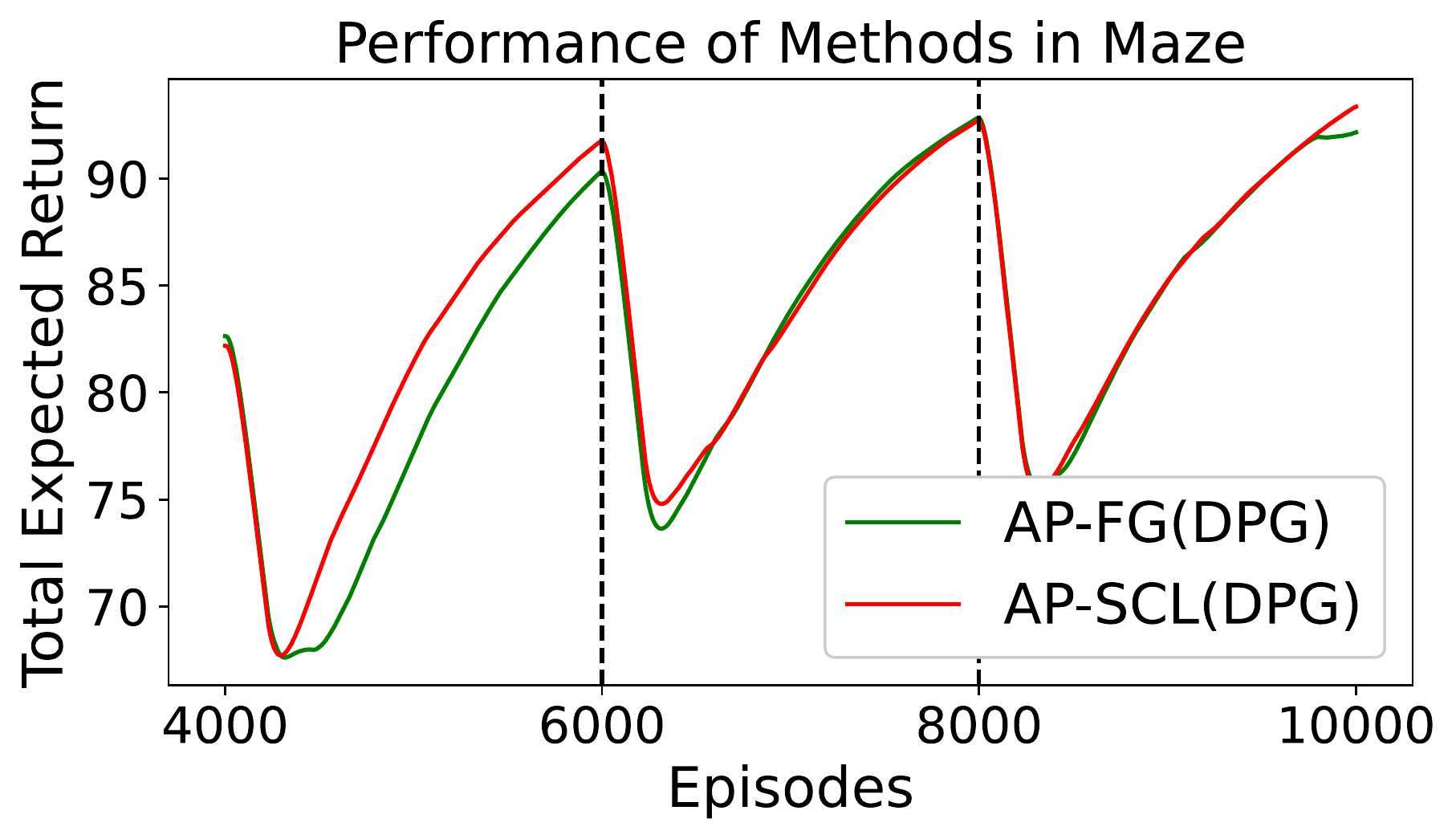}}
	  \\
	\subfigure[]{
		\includegraphics[width=0.49\linewidth]{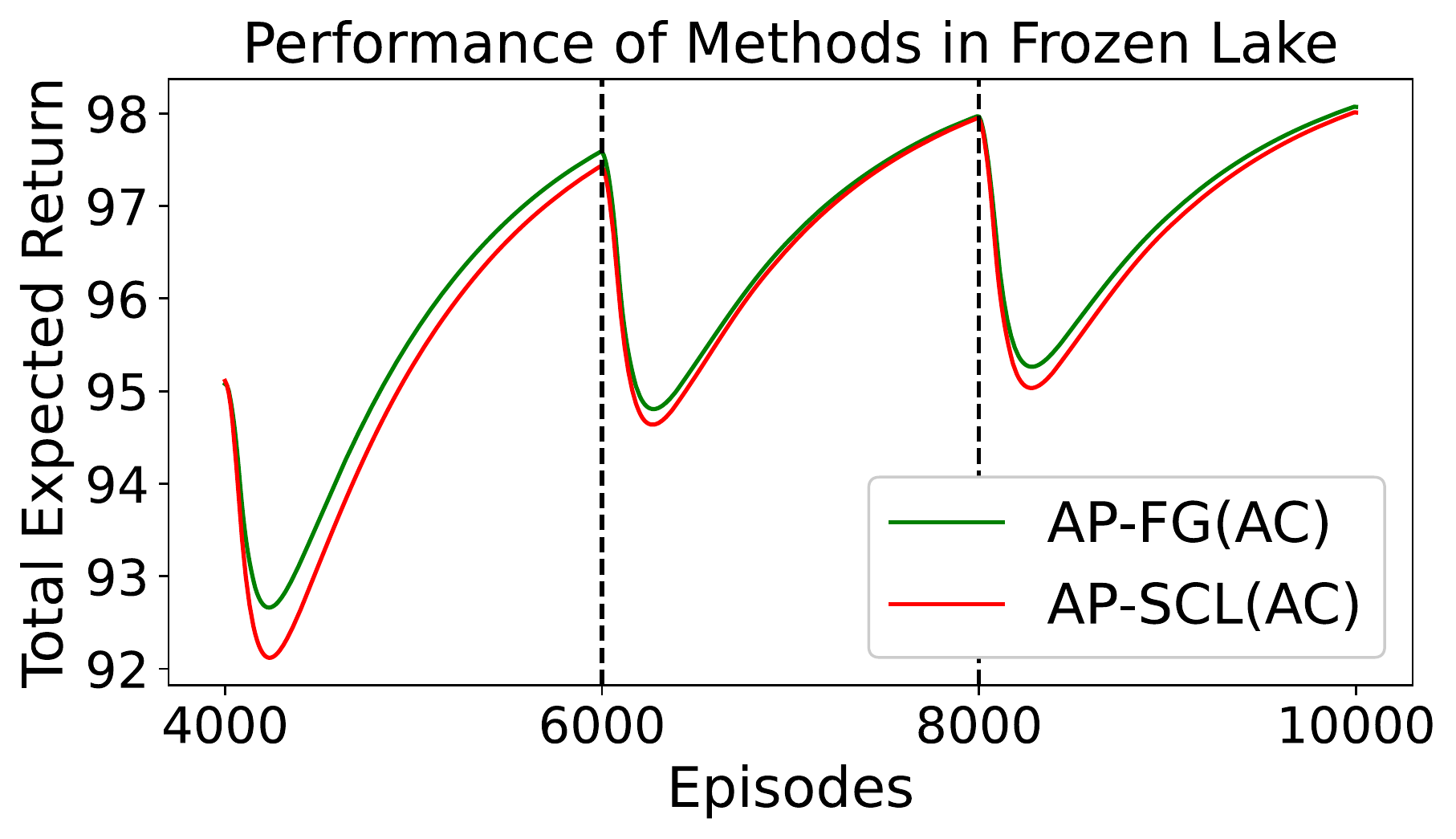}}
	\hspace{-6pt}
	\subfigure[]{
		\includegraphics[width=0.49\linewidth]{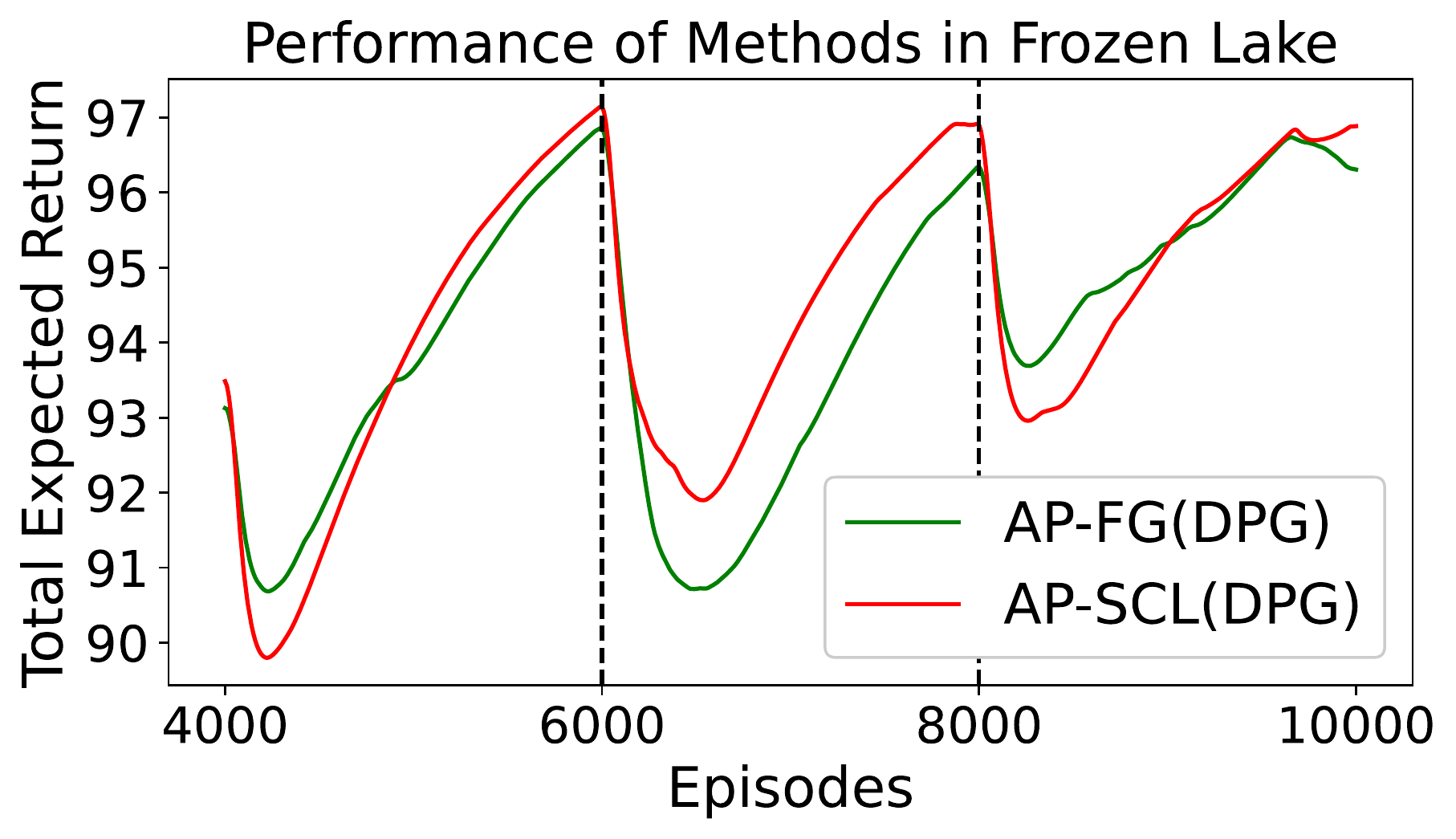}}
	\caption{Partial performance between AP-FG and AP-SCL in the Maze and Frozen Lake environments. The curves correspond to the reward at 4000\~{}10000 episodes.}
        \label{fig:local}
\end{figure}

Comparison results of different methods on the two environments are shown in Figure \ref{fig:all}. In general, the performance of AP methods surpass that of LAICA, showing that AP methods can improve the RL learning efficiency and the performance are more stable. The advantage of our AP methods in the initial phase $\mathcal{M}_0$ due to the removal of repetitive actions. In the later phases of training, the curves of LAICA and our methods are very close. The reason is that, as training episode increases, both LAICA and AP gradually find the optimal policy and the performance reaches saturation.

\begin{figure}[!hbt]
	\centering  
        \setlength{\belowcaptionskip}{-8pt}
	\subfigbottomskip=-3pt 
	\subfigcapskip=-4pt 
	\subfigure[]{
		\includegraphics[width=0.49\linewidth]{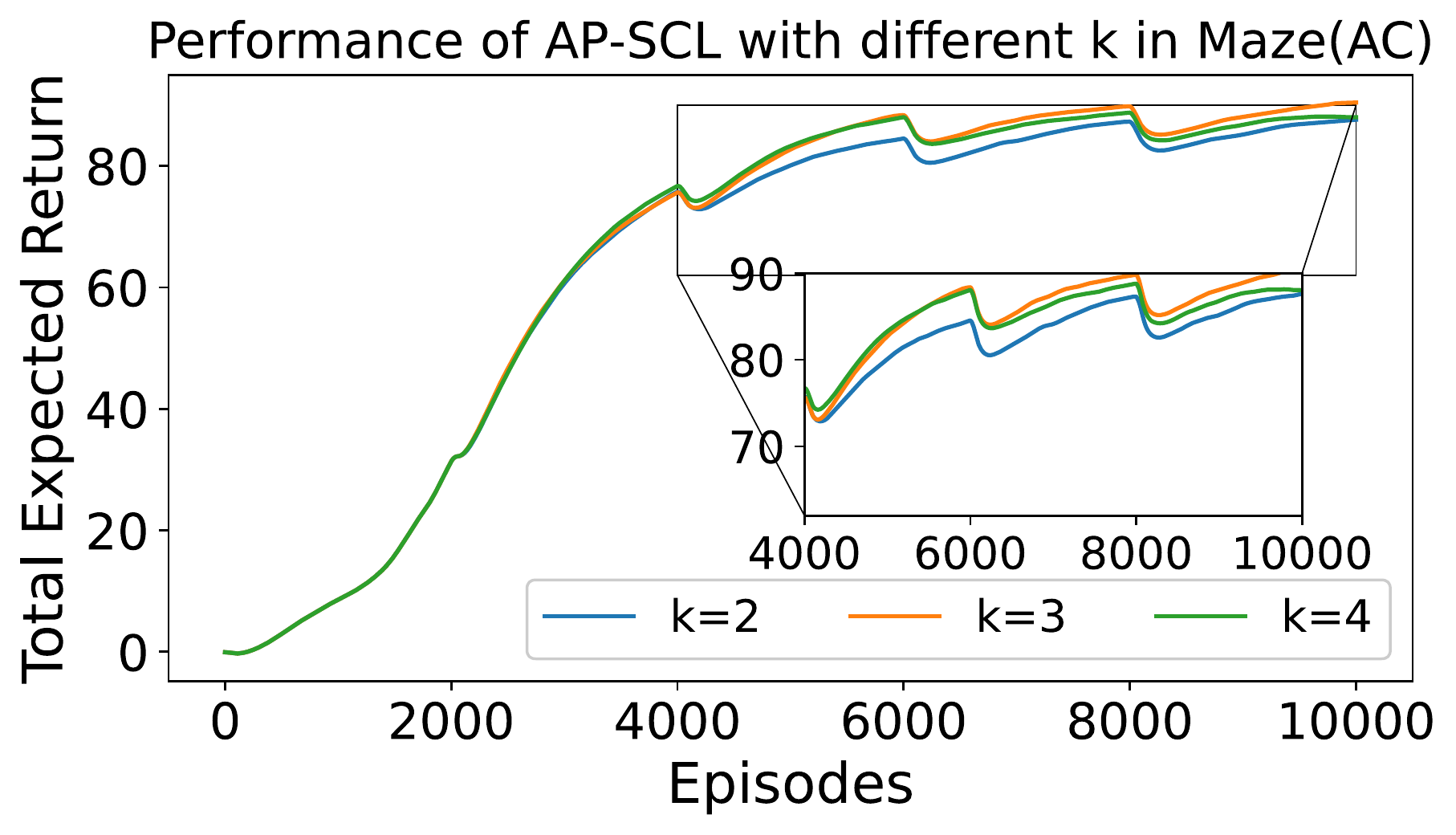}}
        \hspace{-6pt}
	\subfigure[]{
		\includegraphics[width=0.49\linewidth]{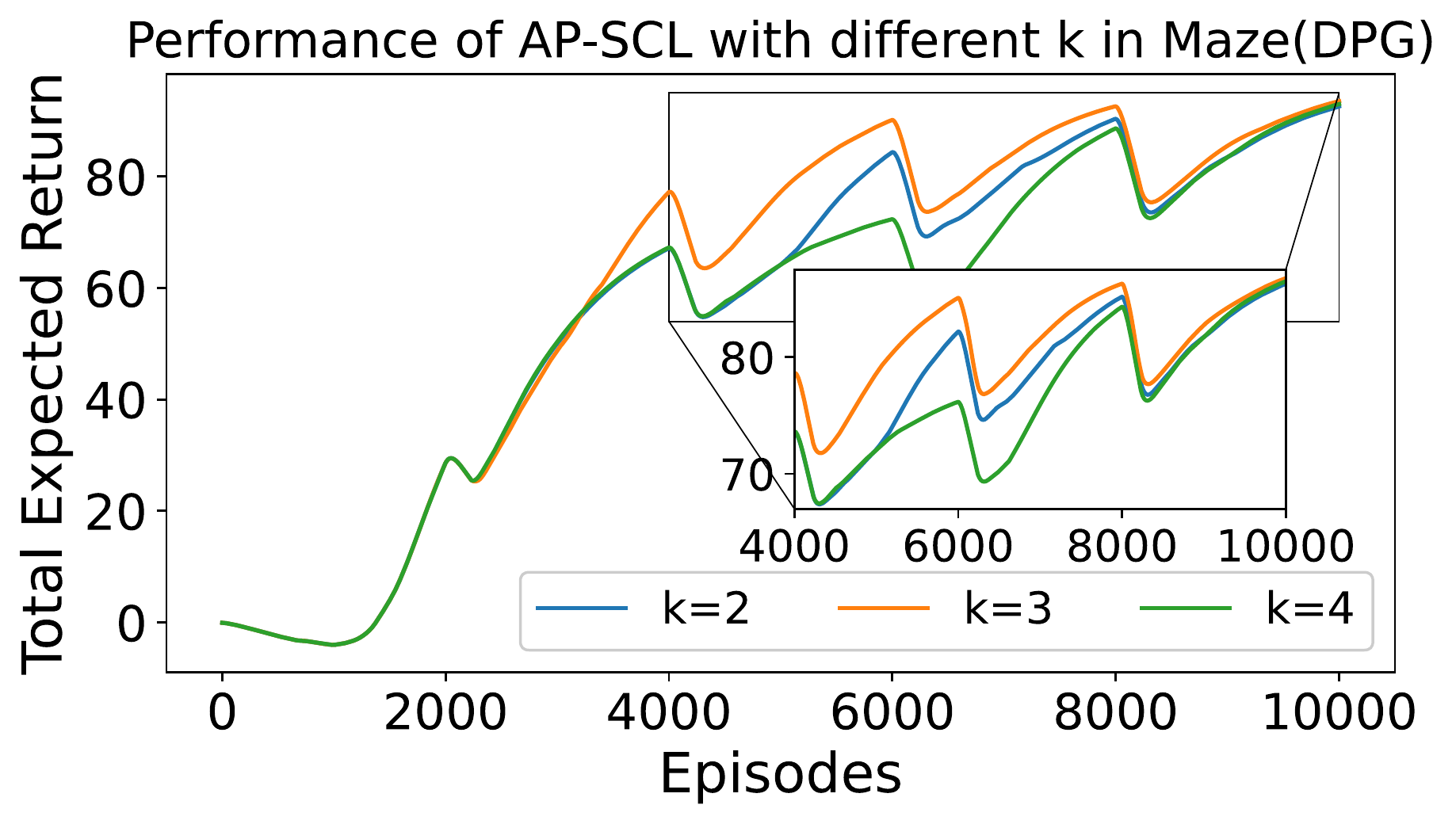}}
	  \\
         \hspace{-6pt}
	\subfigure[]{
		\includegraphics[width=0.49\linewidth]{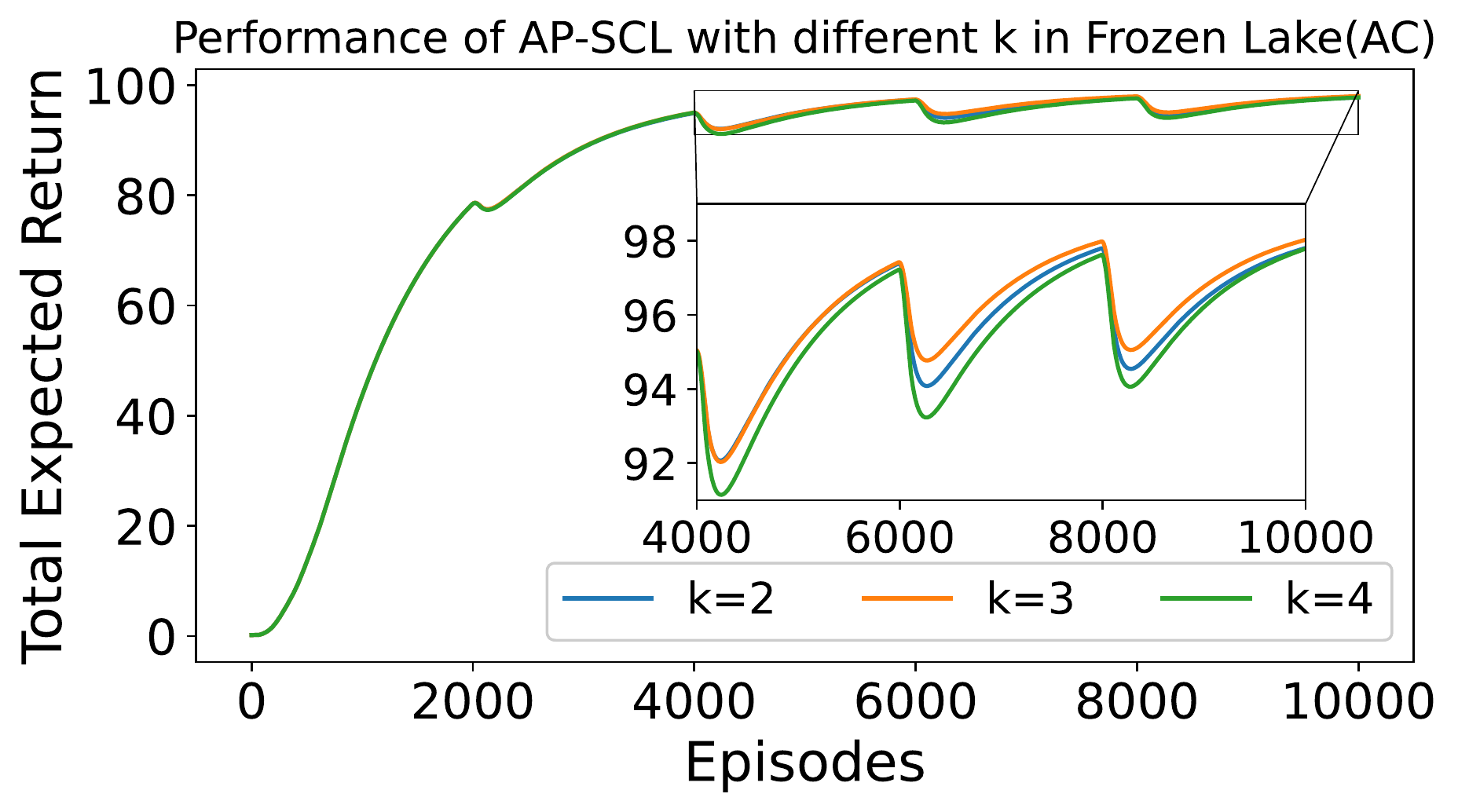}}
	\subfigure[]{
		\includegraphics[width=0.49\linewidth]{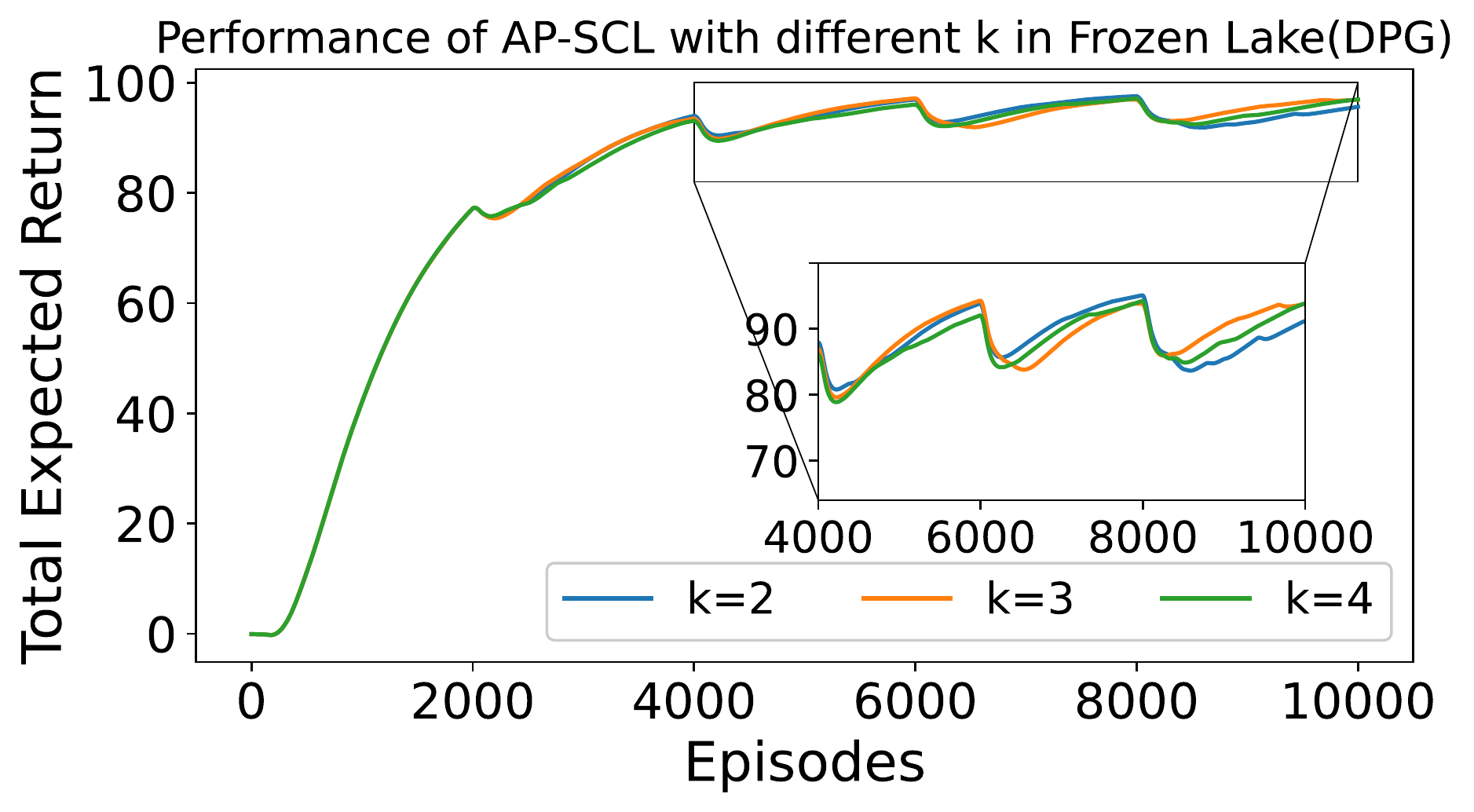}}
	\caption{Performance of AP-SCL method with different $k$ values.}
        \label{fig:kmeans}
\end{figure}

The reason why Random Select method has poor performance is that it randomly selects some actions from new actions, and there is a certain probability that it will omit valuable actions. Adding all the new actions into the original action space in LAICA sacrifices the RL learning efficiency, but such a method avoids neglecting valuable actions and thus guarantees the learning performance. In the AP methods, the learning efficiency is improved while the size of action space is the same as that in Random Select method, which indicates that the actions selected are of great value, and verifies the effectiveness of our AP.

As can be seen from Figure \ref{fig:local}, the AP-SCL has some advantages over the AP-FG in Figures \ref{fig:local}(a), (b) and (d). Table \ref{time} shows the total time required by the AP-FG and the AP-SCL in the action pick-up process. It is obvious that the AP-SCL has advantage over the AP-FG in terms of training time, which also verifies the theoretical analysis in section \ref{sec4.2}. In general, AP-SCL considering both state and action information in the optimal policy is better than the AP-FG.

\begin{table}[htb]
\scriptsize
\renewcommand{\arraystretch}{1.5}
\center
\setlength{\tabcolsep}{5mm}{
\begin{tabular}{cccc}
\toprule
\multirow{2}{*}{RL algorithms} & \multirow{2}{*}{Methods} & \multicolumn{2}{c}{Environments} \\ \cline{3-4} 
  &  & Maze              & Frozen Lake \\ \hline
\multirow{2}{*}{AC}           & AP-FG                   & 813.048           & 121.551     \\
    & AP-SCL     & \textbf{347.561}  & \textbf{49.613}      \\
\multirow{2}{*}{DPG}          & AP-FG                   & 421.694           & 69.634      \\
                              & AP-SCL                  &\textbf{204.072}           & \textbf{48.840}      \\ \bottomrule
\end{tabular}}
\caption{Time consumption\tablefootnote{The unit is millisecond.} 
 during the action pick-up process.}
 \label{time}
\end{table}

In the AP-SCL, we let $k$ equals to 2, 3, and 4 and carry out experiments respectively. The experimental results are shown in Figure \ref{fig:kmeans}. When $k$ equals to 3, the performance is slightly better than others, but there is no obvious advantage.

\section{Conclusion}\label{sec6}
To the best of our knowledge, this is the first work to address the problem of improving the leaning efficiency in DAS-RL. We develop an intelligent action pick-up algorithm to automatically select valuable actions from a set of new actions to improve the RL learning efficiency. In our AP, the prior optimal policy contains useful state and action information, which play an important role in the action pick-up process. Based on the prior optimal policy, we propose two different AP methods: frequency-based global method and state clustering-based local method. Superior performance on Maze and Frozen Lake environments demonstrate that our AP can effectively accelerate the convergence speed of RL.

In future work, we intend to apply this algorithm to a wider range of real-world control scenarios. For some complex realistic scenarios, such as recommender systems, we will directly leverage neural networks to predict the value of each action and select valuable actions instead of solving the objective function.


\bibliographystyle{named}
\bibliography{ijcai23}

\end{document}